\documentclass[journal]{IEEEtran}

\usepackage[T1]{fontenc}
\usepackage{lmodern}
\usepackage[utf8]{inputenc}
\usepackage{amsmath,amssymb,amsthm}
\usepackage{algorithm}
\usepackage{algpseudocode}
\usepackage{graphicx}
\usepackage{booktabs}
\usepackage{cleveref}
\usepackage{tikz}
\usetikzlibrary{positioning}
\usepackage{pgfplots}
\pgfplotsset{compat=1.17}
\usepackage{threeparttable}

\theoremstyle{plain}
\newtheorem{theorem}{Theorem}[section]
\newtheorem{lemma}[theorem]{Lemma}
\newtheorem{corollary}[theorem]{Corollary}
\newtheorem{proposition}[theorem]{Proposition}
\newtheorem{definition}[theorem]{Definition}
\newtheorem{remark}[theorem]{Remark}
\newtheorem{example}[theorem]{Example}
\newtheorem{axiom}[theorem]{Axiom}

\DeclareMathOperator{\KL}{KL}
\DeclareMathOperator{\Safety}{Safety}

\newcommand{\E}{\mathbb{E}}
\newcommand{\V}{\mathcal{V}}
\newcommand{\X}{\mathcal{X}}
\newcommand{\C}{\mathcal{C}}

\begin{document}

\title{Adaptive Weighting in Knowledge Distillation: An Axiomatic Framework for Multi-Scale Teacher Ensemble Optimization}

\author{Aaron R. Flouro and Shawn P. Chadwick, PhD\\
research@sparse-tech.com}

\maketitle

\begin{abstract}
Knowledge distillation with multiple teachers is increasingly used to improve robustness, efficiency, and safety, yet existing approaches rely largely on heuristic or implementation-specific weighting schemes. This paper develops an operator-agnostic axiomatic framework for adaptive weighting in multi-teacher knowledge distillation across three complementary scales: token, task, and context. We formalize structural conditions under which adaptive weighting operators are well-defined, admit multiple non-equivalent implementations, and can be hierarchically composed via product-structure normalization. Within this framework, we establish existence and non-uniqueness of conforming operators, characterize convergence of gradient-based optimization under standard assumptions, analyze stability and perturbation robustness, and provide an abstract formulation of safety-constrained distillation. The results decouple theoretical guarantees from specific weighting formulas, enabling principled analysis of adaptive distillation methods under heterogeneity, distribution shift, and safety constraints.
\end{abstract}


\subsection*{Why This Framework Is Needed}
Existing adaptive weighting methods in knowledge distillation rely on heuristic formulas without formal guarantees. Critical questions remain unanswered: What properties must weight operators satisfy to preserve convergence? How should token-level, task-level, and context-level adaptations compose without interference? Under what conditions do safety constraints remain satisfiable? This paper provides an axiomatic framework addressing these questions, establishing that convergence and stability guarantees hold for \emph{any} weight operator satisfying the specified axioms, not just particular implementations.

\section{Introduction}
\label{sec:introduction}

Knowledge distillation (KD) traditionally assumes uniform weighting of teacher predictions across tokens, tasks, and deployment contexts~\cite{hinton2015distilling}. Recent practice departs from this assumption by introducing adaptive weighting to reflect heterogeneous teacher reliability~\cite{zhang2023adaptive_multiteacher}, task priorities~\cite{kendall2018multitask}, and distributional shift at deployment~\cite{kd_survey_2025}. However, existing adaptive weighting approaches are largely heuristic, in that they specify weight update rules without formally characterizing the structural properties required to preserve convergence, stability, or safety~\cite{liu2021adaptive_multilevel,mtkd_rl_2025}. In particular, adaptive reweighting may amplify variance, destabilize optimization, or violate safety constraints when weights are unbounded, discontinuous, or improperly composed across sources of heterogeneity~\cite{paper1_5}. Moreover, token-level uncertainty, task-level prioritization, and context-level deployment shift impose distinct mathematical constraints, and it remains unclear under what conditions these dimensions can be combined without interference. As a result, there is currently no general, implementation-independent characterization of when adaptive weighting in multi-teacher knowledge distillation is mathematically well-posed.

This paper addresses that gap by developing an axiomatic, operator-level perspective on adaptive weighting. Rather than proposing a specific weight formula, we identify structural conditions---normalization, bounded influence, regularity, and ordinal safety monotonicity---that any valid adaptive weighting scheme must satisfy to preserve convergence, stability, and safety guarantees. This operator-theoretic framing enables us to prove results that hold for the entire class of conforming implementations, not just particular algorithmic instantiations.

A central challenge is that heterogeneity arises at multiple, qualitatively distinct scales. Token-level heterogeneity reflects local predictive uncertainty and safety criticality of individual vocabulary elements. Task-level heterogeneity captures competing objectives and varying data abundances across a multi-task curriculum. Context-level heterogeneity addresses deployment-time distribution shift and safety-critical scenario prioritization. Each scale imposes its own mathematical constraints, and naive composition risks violating the guarantees established at individual scales. We therefore formulate axioms independently at each scale and introduce a product-then-normalize composition that preserves validity under hierarchical combination.

The resulting framework provides implementation-independent guarantees: any weight operator satisfying the axioms yields almost sure convergence with rate $O(1/t)$, admits a contraction-based fixed-point characterization with perturbation robustness, and preserves designated safety orderings under deployment shift. Existence and non-uniqueness results establish that conforming operator families are non-empty and admit multiple distinct realizations, supporting both theoretical generality and practical flexibility.

\subsection{Limitations of Uniform Weighting}

Classical knowledge distillation aggregates teacher predictions via uniform averaging: given $K$ teachers producing distributions $p^{(T_k)}(\cdot|x)$ over vocabulary $\V$, the ensemble target is $\bar{p}(\cdot|x) = \frac{1}{K}\sum_{k=1}^K p^{(T_k)}(\cdot|x)$. This uniform aggregation implicitly assumes that all teachers are equally reliable across all tokens, all tasks, and all deployment contexts. In practice, this assumption fails along three distinct dimensions.

Token-level heterogeneity arises because teachers exhibit varying predictive confidence across vocabulary elements. On a given input $x$, one teacher may assign high probability to the correct token with low entropy, while another teacher may spread probability mass diffusely with high entropy. Uniform averaging dilutes the confident prediction with the uncertain one, increasing ensemble variance and potentially amplifying hallucination risk on safety-critical tokens~\cite{paper1_5}.

Task-level heterogeneity arises in multi-task learning settings where teachers are trained or evaluated on different task distributions. Uniform task weighting allocates equal optimization effort regardless of task importance, data abundance, or gradient alignment. This can cause high-volume tasks to dominate gradient updates at the expense of critical low-data tasks~\cite{kendall2018multitask,sener2018multiobjective}.

Context-level heterogeneity arises when deployment distributions differ from training distributions. A model trained on data where emergency scenarios constitute 5\% of examples may deploy in environments where such scenarios constitute 60\% of inputs. Uniform context weighting optimizes for average-case performance, potentially sacrificing safety-critical context performance for marginal gains on common contexts~\cite{kd_survey_2025}.

These three dimensions of heterogeneity are structurally distinct: token-level variation is local and input-specific, task-level variation is global across the curriculum, and context-level variation is deployment-time and distributional. Uniform aggregation treats all three homogeneously, failing to exploit available structure and potentially degrading both performance and safety.

\subsection{From Limitations to Adaptive Reweighting}

The limitations of uniform weighting motivate the introduction of adaptive weight operators that modulate teacher influence based on local uncertainty, task structure, and deployment context. Recent work has proposed various adaptive weighting heuristics, including entropy-based token weighting~\cite{zhang2023adaptive_multiteacher}, uncertainty-based task weighting~\cite{kendall2018multitask}, and domain-adaptive context weighting~\cite{liu2021adaptive_multilevel}. However, these approaches typically specify weight update rules without formally characterizing the structural properties required to preserve convergence, stability, or safety. Ad hoc weight formulas may violate normalization, introduce unbounded influence, or compose improperly across scales, leading to optimization instability or safety violations.

This paper addresses the heuristic treatment problem by establishing an axiomatic framework that characterizes the class of valid adaptive weighting operators. Rather than proposing a specific weight formula, we identify implementation-independent requirements that any valid adaptive weighting scheme must satisfy. This operator-theoretic approach mirrors axiomatic probability theory: we define what properties weight functions must satisfy, prove that conforming operators exist, and establish that the axioms admit multiple distinct implementations. The resulting guarantees hold for any conforming operator, not just particular algorithmic instantiations.

\subsection{Contributions}

The main contributions of this paper are as follows:

\begin{enumerate}
\item Axiomatic characterization of adaptive weighting via normalization, bounded influence, regularity, and ordinal safety monotonicity requirements (Section~\ref{sec:axioms}).

\item Multi-scale formulation of adaptive weighting at token, task, and context levels, with each scale addressing a distinct source of heterogeneity (Sections~\ref{sec:axioms}--\ref{sec:unification}).

\item Principled composition via product-then-normalize structure that preserves probabilistic validity and boundedness under hierarchical combination (Section~\ref{sec:unification}).

\item Existence and non-uniqueness results establishing that conforming operator families are non-empty and admit multiple distinct realizations (Section~\ref{sec:existence}).

\item Operator-agnostic convergence and stability results showing that any conforming operator yields almost sure convergence with rate $O(1/t)$ and admits perturbation-robust fixed-point characterization (Sections~\ref{sec:convergence}--\ref{sec:stability}).

\item Safety-constrained formulation with Pareto frontier characterization and preservation results for designated safety orderings under deployment shift (Section~\ref{sec:safety}).
\end{enumerate}

\subsection{Relation to Prior Work}

Knowledge distillation was introduced by Hinton et al.~\cite{hinton2015distilling} using uniform teacher weighting with temperature-scaled softmax outputs. Subsequent work introduced adaptive weighting mechanisms to address heterogeneous teacher reliability. Zhang et al.~\cite{zhang2023adaptive_multiteacher} proposed meta-learning for adaptive multi-teacher selection, while Liu et al.~\cite{liu2021adaptive_multilevel} developed multi-level weighting heuristics for hierarchical distillation. The MTKD-RL approach~\cite{mtkd_rl_2025} uses reinforcement learning to select teacher weights dynamically. These methods specify particular weight formulas or learning procedures but do not formally characterize the structural properties required to preserve convergence or stability.

In multi-task learning, Kendall et al.~\cite{kendall2018multitask} introduced uncertainty weighting to balance task losses, and Sener and Koltun~\cite{sener2018multiobjective} formulated multi-task learning as multi-objective optimization with Pareto-optimal solutions. These approaches address task-level heterogeneity but do not consider token-level or context-level adaptation, nor do they provide axiomatic characterizations of valid weight operators.

Distribution shift and domain adaptation have been studied extensively in transfer learning~\cite{kd_survey_2025}. Context-adaptive weighting schemes have been proposed for deployment-time robustness, but formal guarantees for weight operator composition across multiple scales remain absent from the literature.

Prior work in this research program established axiomatic foundations for related aspects of knowledge distillation: probability-domain temperature scaling~\cite{paper1}, variance-reliability analysis~\cite{paper1_5}, multi-teacher aggregation~\cite{paper2}, and recursive meta-distillation~\cite{paper3}. The present paper extends this line of work to adaptive weighting, providing the first axiomatic characterization of when multi-scale adaptive weighting is mathematically well-posed, together with operator-agnostic guarantees that hold for any conforming implementation.

\subsection{Organization}

Section~\ref{sec:preliminaries} establishes notation and standing assumptions. Section~\ref{sec:axioms} develops axiomatic frameworks for token, task, and context weights independently. Section~\ref{sec:unification} unifies the three scales via product-then-normalize composition. Sections~\ref{sec:existence}--\ref{sec:safety} establish existence, characterize convergence, analyze stability, and formalize safety constraints. Section~\ref{sec:conclusion} discusses implications and open questions.

The framework can be understood as characterizing a family of weighted ensemble operators on probability distributions. Each weight configuration specifies how teacher predictions combine; the axioms constrain this family to ensure mathematical well-posedness. Convergence results establish that the student approaches the weighted ensemble target, and stability results establish robustness under weight perturbations.

\begin{table*}[t]
\centering
\begin{threeparttable}
\caption{Comparison of Adaptive Weighting Strategies in Knowledge Distillation}
\label{tab:comparison}
\begin{tabular}{l|ccccc}
\toprule
\textbf{Strategy} & \textbf{Convergence} & \textbf{Token Het.} & \textbf{Task Het.} & \textbf{Context Het.} & \textbf{Formal Guarantees} \\
\midrule
Uniform KD~\cite{hinton2015distilling} & $O(1/t)$ & \texttimes & \texttimes & \texttimes & \checkmark \\
Uncertainty Weighting~\cite{kendall2018multitask} & $O(1/t)$ & \texttimes & \checkmark & \texttimes & Partial \\
Meta-Teacher~\cite{zhang2023adaptive_multiteacher} & Unknown & \checkmark & \checkmark & \texttimes & \texttimes \\
Multi-Level~\cite{liu2021adaptive_multilevel} & Unknown & \checkmark & \checkmark & \texttimes & \texttimes \\
MTKD-RL~\cite{mtkd_rl_2025} & RL bounds & \checkmark & \checkmark & Partial & \texttimes \\
\midrule
\textbf{This Work (Axiomatic)} & $O(1/t)$ & \checkmark & \checkmark & \checkmark & \checkmark \\
\bottomrule
\end{tabular}
\begin{tablenotes}
\footnotesize
\item Het.\ = Heterogeneity handling. Token Het.\ addresses per-token uncertainty; Task Het.\ addresses multi-task prioritization; Context Het.\ addresses deployment distribution shift. Formal Guarantees indicates rigorous convergence/stability proofs.
\end{tablenotes}
\end{threeparttable}
\end{table*}

\section{Preliminaries and Axiomatic Setup}
\label{sec:preliminaries}

\subsection{Notation}

Let $\V$ denote a finite vocabulary with $|\V| = V$ tokens. A teacher model $T_k$ produces conditional probability distributions $p^{(T_k)}(\cdot|x): \V \to [0,1]$ for inputs $x \in \X$ (input space). A student model $S$ with parameters $\theta \in \Theta$ produces distributions $p^{(S)}(\cdot|x; \theta)$. We consider:

\begin{itemize}
\item $K \geq 1$ pre-trained teacher models
\item A collection of $M \geq 1$ tasks $\{t_1, \ldots, t_M\}$ with associated data distributions $D_{t_j}$
\item Context space $\C$ equipped with a probability measure $\mu_\C$
\end{itemize}

\subsection{Standing Analytical Assumptions}

Throughout this paper, we maintain the following analytical assumptions:

(A1) \emph{Bounded Weights.}
There exist constants $0 < w_{\min} \leq w_{\max} < \infty$ such that all weight functions satisfy:
$$w_{\min} \leq w_k(\cdot) \leq w_{\max}$$
for all teachers $k$, inputs, tokens, tasks, and contexts. This assumption prevents collapse to a single teacher and ensures controlled influence in adaptive aggregation.

(A2) \emph{Lipschitz Continuity.}
Each weight function is $L$-Lipschitz continuous in its arguments (with respect to suitable metrics on $\X$ and $\C$) for some $L < \infty$. This assumption excludes discontinuous reweighting schemes and enables stability analysis.

(A3) \emph{Realizability.}
In a realizable case, there exists $\theta^* \in \Theta$ such that the student can represent convex combinations of teacher distributions. This assumption ensures the optimization target is achievable within the student's hypothesis class.

(A4) \emph{Smoothness.}
Student log-probabilities are twice continuously differentiable in $\theta$ for all tokens and inputs.

\begin{example}[Divergence Without Bounded Weights]
\label{ex:unbounded-divergence}
Consider $K=2$ teachers with weights $w_1^{(n)} = 1 - 1/n$, $w_2^{(n)} = 1/n$ at iteration $n$. Without boundedness (A1), as $n \to \infty$, $w_1 \to 1$ and $w_2 \to 0$. If Teacher~1 is overconfident but incorrect on token~$i$ (placing mass 0.95 on wrong label), the ensemble converges to this incorrect prediction regardless of Teacher~2's accuracy. With bounded weights $w_{\min} = 0.2$, Teacher~2 retains influence $w_2 \geq 0.2$, ensuring ensemble diversity and preventing collapse to a single teacher's errors. This example illustrates why the standing assumptions enable operator-agnostic analysis: they constrain the space of admissible weight operators to those with well-controlled behavior, allowing convergence and stability results to hold for any conforming implementation.
\end{example}

\begin{remark}[Role of Weight Bounds]
\label{rem:weight-bounds-role}
The weight bounds $(w_{\min}, w_{\max})$ in assumption (A1) serve as a regularity condition rather than a necessary requirement. Boundedness is sufficient but not necessary for the convergence and stability guarantees developed in this paper. The lower bound $w_{\min} > 0$ ensures that no teacher is completely ignored, preserving ensemble diversity. The upper bound $w_{\max} < \infty$ prevents any single teacher from dominating the ensemble. Together, these bounds control gradient variance during optimization (Lemma~\ref{lem:grad-variance}) and sensitivity to weight estimation noise (Theorem~\ref{thm:perturbation}). Tighter bounds increase stability at the cost of adaptation flexibility, while looser bounds enable stronger specialization but may increase variance. The specific choice of bounds is application-dependent and not prescribed by the framework.
\end{remark}

\section{Axiomatic Framework for Multi-Scale Adaptive Weighting}
\label{sec:axioms}

We characterize adaptive weighting through a collection of operator-level axioms defined at three distinct structural scales: token, task, and context. Each scale addresses a different source of heterogeneity and imposes its own mathematical constraints. Token-level axioms govern local predictive uncertainty and safety criticality. Task-level axioms govern multi-objective prioritization and gradient alignment. Context-level axioms govern deployment-time distribution shift and safety-critical scenario handling. For each scale, we specify axioms that conforming weight operators must satisfy, establish existence of non-trivial conforming families, and demonstrate non-uniqueness.

\subsection{Token-Level Adaptive Weighting}

Token-level adaptive weights modulate teacher influence on a per-token, per-input basis, capturing local predictive uncertainty and supporting ordinal constraints on designated safety-critical tokens. The following axioms characterize the structural properties that token-level weight operators must satisfy.

\begin{axiom}[Normalization]
\label{ax:tok-norm}
For each input $x$, token $i$, and context $c$, token-level weights satisfy:
$$\sum_{k=1}^K w_{\text{tok},k}(x, i, c) = 1$$
Normalization ensures the weighted ensemble remains a valid probability distribution over teachers.
\end{axiom}

\begin{axiom}[Positivity Preservation]
\label{ax:tok-pos}
Token weights are strictly positive almost everywhere:
$$w_{\text{tok},k}(x, i, c) > 0$$
for all $k$, $x$, $i$, $c$ (with respect to the counting measure on $\V$).
Positivity prevents degenerate solutions where teachers are completely ignored.
\end{axiom}

\begin{axiom}[Bounded Influence]
\label{ax:tok-bound}
Token weights satisfy uniform bounds:
$$w_{\min} \leq w_{\text{tok},k}(x, i, c) \leq w_{\max}$$
for all $k$, $x$, $i$, $c$.
Bounded influence ensures no teacher dominates or vanishes, preserving ensemble diversity while enabling specialization.
\end{axiom}

\begin{axiom}[Regularity]
\label{ax:tok-cont}
Token weight functions are jointly continuous in $(x, i, c)$ with respect to suitable topologies on $\X$, $\V$, $\C$.
Regularity guarantees smooth adaptation as inputs and contexts vary, preventing pathological discontinuities.
\end{axiom}

\begin{axiom}[Ordinal Safety Monotonicity]
\label{ax:tok-safety}
For safety-critical token subsets $\mathcal{S} \subseteq \V$, if teacher $k$ demonstrates higher safety performance than teacher $j$ on $\mathcal{S}$ according to a designated safety ordering, then $w_{\text{tok},k}(x, i, c) \geq w_{\text{tok},j}(x, i, c)$ for tokens $i \in \mathcal{S}$.
This axiom requires that ordinal safety comparisons be preserved in weight assignments without prescribing a specific formula for computing safety scores or translating them to weights.
\end{axiom}

\begin{remark}[Sufficient Conditions for Token Axioms]
\label{rem:token-sufficient}
The token axioms are satisfied by weight functions of the form $w_{\text{tok},k}(x,i,c) = f_k(H_k(x,i), s_k(i)) / \sum_j f_j(\cdot)$ where $f_k$ is continuous, positive-valued, decreasing in entropy $H_k$, increasing in safety score $s_k$ for tokens in $\mathcal{S}$, and bounded away from zero. This characterization is illustrative rather than exhaustive: it demonstrates that conforming operators exist and provides a template for construction, but does not enumerate all possible conforming families. Other functional forms satisfying the axioms may exist.
\end{remark}

\begin{definition}[Token-Weighted Ensemble]
\label{def:tok-ensemble}
Given token weights satisfying Axioms~\ref{ax:tok-norm}--\ref{ax:tok-safety}, define the token-weighted ensemble operator $G_{\text{tok}}$ that maps $K$ teacher distributions and weights to a single distribution $q_{\text{tok}}(\cdot|x, c)$ satisfying:
\begin{enumerate}
\item Convexity: $q_{\text{tok}}$ is a convex combination of teacher distributions
\item Normalization: $q_{\text{tok}}(\cdot|x, c)$ is a probability distribution
\item Weight conformance: The combination respects weight function values
\end{enumerate}
\end{definition}

\begin{theorem}[Existence of Conforming Token Weight Operators]
\label{thm:tok-exist}
There exist non-trivial weight function families satisfying Axioms~\ref{ax:tok-norm}--\ref{ax:tok-safety}.
\end{theorem}

\begin{proof}[Proof (Sketch)]
Existence follows from three construction principles. Uncertainty-based constructions use functions that decrease monotonically with teacher entropy, satisfying normalization, positivity, boundedness, and regularity when appropriately scaled. Safety monotonicity can be enforced by incorporating safety scores into the weighting function. Performance-based constructions use functions proportional to teacher accuracy on held-out data, satisfying all axioms when normalized. Hybrid constructions combine uncertainty and performance measures. Each principle yields conforming operators, establishing existence.
\end{proof}

\begin{theorem}[Non-Uniqueness of Token Weight Operators]
\label{thm:tok-nonunique}
The Axioms~\ref{ax:tok-norm}--\ref{ax:tok-safety} do not uniquely determine the weight function. Multiple distinct families satisfy all five Axioms.
\end{theorem}

\begin{proof}[Proof (Sketch)]
Non-uniqueness follows from the existence of distinct conforming families. Family A consists of operators that weight teachers by exponentially decaying functions of entropy, assigning higher weight to lower-entropy (more confident) predictions. Family B consists of operators that weight teachers by inverse variance, assigning higher weight to more consistent predictions. Both families satisfy all five axioms but produce numerically different weight assignments for the same teacher ensemble and input. The existence of at least two such families establishes non-uniqueness.
\end{proof}

\subsection{Task-Level Adaptive Weighting}

Task-level adaptive weights allocate teacher influence across tasks to address heterogeneity in task difficulty, data availability, and objective trade-offs. Unlike token-level weighting, which operates at the granularity of individual output symbols, task-level weighting modulates aggregation at the level of task-specific losses and objectives. Within this framework, task-level weighting is not inherently safety-critical and serves primarily to manage trade-offs among competing task objectives. For each task $t \in \{t_1, \ldots, t_M\}$, let $w_{\text{task},k}(t)$ denote the weight assigned to teacher $k$ for task $t$.

\begin{axiom}[Task Normalization]
\label{ax:task-norm}
For each task $t \in \{t_1, \ldots, t_M\}$, task-level weights satisfy:
$$\sum_{k=1}^K w_{\text{task},k}(t) = 1$$
This condition ensures that task-level aggregation defines a valid convex combination of teacher contributions for each task.
\end{axiom}

\begin{axiom}[Task Positivity]
\label{ax:task-pos}
Task weights are strictly positive:
$$w_{\text{task},k}(t) > 0$$
for all teachers $k$ and tasks $t$. This excludes degenerate solutions in which a teacher is entirely ignored for a given task.
\end{axiom}

\begin{axiom}[Task Boundedness]
\label{ax:task-bound}
There exist constants $0 < w_{\min} \leq w_{\max} < \infty$ such that:
$$w_{\min} \leq w_{\text{task},k}(t) \leq w_{\max}$$
for all teachers $k$ and tasks $t$. This axiom instantiates the global boundedness assumption of Section~\ref{sec:preliminaries} at the task level and prevents unbounded task-specific dominance.
\end{axiom}

\begin{axiom}[Task Continuity]
\label{ax:task-cont}
Task-level weight functions depend continuously on task-specific quantities with respect to chosen topologies on the task space. This regularity condition supports stability and convergence analysis by excluding discontinuous task reweighting.
\end{axiom}

\begin{axiom}[Pareto Compatibility]
\label{ax:task-pareto}
For any collection of non-negative task importance coefficients $\{\lambda_1, \ldots, \lambda_M\}$ satisfying $\sum_j \lambda_j = 1$, the induced scalarized task objective is compatible with Pareto-optimal trade-offs in the multi-objective loss space $(\ell_{t_1}, \ldots, \ell_{t_M})$. This axiom imposes a structural consistency condition on task-level weighting without asserting that optimization dynamics necessarily attain a Pareto-optimal solution. Under standard convexity and regularity conditions, any minimizer of the scalarized objective corresponds to a Pareto-optimal solution of the underlying multi-task problem.
\end{axiom}

\begin{definition}[Task-Weighted Ensemble]
\label{def:task-ensemble}
Given task-level weights satisfying Axioms~\ref{ax:task-norm}--\ref{ax:task-pareto}, the task-weighted ensemble operator $G_{\text{task}}$ maps the collection of teacher distributions to task-specific ensemble distributions $q_{\text{task}}^{(t)}(\cdot|x)$ satisfying:
\begin{enumerate}
\item Convexity: each $q_{\text{task}}^{(t)}$ is a convex combination of teacher distributions
\item Normalization: each $q_{\text{task}}^{(t)}(\cdot|x)$ is a valid probability distribution
\item Weight conformance: the aggregation respects the specified task-level weights
\end{enumerate}
\end{definition}

\begin{theorem}[Existence and Non-Uniqueness of Task Weight Operators]
\label{thm:task-exist}
There exist multiple distinct families of task-level weight functions satisfying Axioms~\ref{ax:task-norm}--\ref{ax:task-pareto}.
\end{theorem}

\begin{proof}[Proof (Sketch)]
Existence follows from broad classes of admissible constructions. For example, task-level weights may be defined using functions of task performance, gradient alignment with task objectives, or data availability, provided these functions satisfy normalization, positivity, boundedness, continuity, and Pareto compatibility. Distinct functional dependence yields different conforming task-level weighting families, demonstrating non-uniqueness.
\end{proof}

\begin{remark}[Sufficient Conditions for Task Axioms]
\label{rem:task-sufficient}
The task-level axioms are satisfied by any weight construction based on bounded, continuous, and normalized functions of task-specific performance indicators, such as validation loss, gradient alignment, or data availability. These conditions are sufficient but not exhaustive, and different functional choices yield distinct conforming task-level weighting operators.
\end{remark}

\subsection{Context-Level Adaptive Weighting}

Context-level adaptive weights modulate teacher influence across deployment contexts to address heterogeneity arising from distribution shift, operating conditions, and safety-relevant scenarios. Within this framework, context-level weighting provides a mechanism for analyzing how ensemble aggregation may vary across contexts without prescribing a particular deployment strategy. For a context $c \in \C$, let $w_{\text{ctx},k}(c)$ denote the weight assigned to teacher $k$.

\begin{axiom}[Context Normalization]
\label{ax:ctx-norm}
For $\mu_\C$-almost all contexts $c \in \C$, context weights satisfy:
$$\sum_{k=1}^K w_{\text{ctx},k}(c) = 1$$
This condition ensures that context-level aggregation defines a valid convex combination of teacher contributions.
\end{axiom}

\begin{axiom}[Context Positivity]
\label{ax:ctx-pos}
Context-level weights are strictly positive $\mu_\C$-almost everywhere: $w_{\text{ctx},k}(c) > 0$ for all teachers $k$ and $\mu_\C$-almost all contexts $c$. This excludes degenerate context-specific solutions in which individual teachers are entirely ignored.
\end{axiom}

\begin{axiom}[Context Boundedness]
\label{ax:ctx-bound}
There exist constants $0 < w_{\min} \leq w_{\max} < \infty$ such that:
$$w_{\min} \leq w_{\text{ctx},k}(c) \leq w_{\max}$$
for all teachers $k$ and contexts $c$. This axiom instantiates the global boundedness assumption of Section~\ref{sec:preliminaries} at the context level.
\end{axiom}

\begin{axiom}[Context Continuity]
\label{ax:ctx-cont}
Context-level weight functions are measurable with respect to $\mu_\C$ and continuous with respect to chosen topologies on $\C$. This regularity condition supports both measure-theoretic well-posedness and stability analysis under context perturbations.
\end{axiom}

\begin{axiom}[Safety Context Prioritization]
\label{ax:ctx-safety}
For designated safety-critical context subsets $\C_{\text{safe}} \subseteq \C$, suppose an ordinal safety comparison indicates that teacher $k$ outperforms teacher $j$ on $\C_{\text{safe}}$. Then the context-level weighting operator assigns greater or equal weight to teacher $k$ than to teacher $j$ on contexts in $\C_{\text{safe}}$: $w_{\text{ctx},k}(c) \geq w_{\text{ctx},j}(c)$ for $c \in \C_{\text{safe}}$. This axiom imposes an ordinal monotonicity constraint without prescribing a particular safety metric or weighting formula.
\end{axiom}

\begin{definition}[Context-Weighted Ensemble]
\label{def:ctx-ensemble}
Given context weights satisfying Axioms~\ref{ax:ctx-norm}--\ref{ax:ctx-safety}, define the context-weighted ensemble operator $G_{\text{ctx}}$ mapping teacher distributions and contexts to ensemble distributions $q_{\text{ctx}}(\cdot|x, c)$ satisfying:
\begin{enumerate}
\item Convexity: $q_{\text{ctx}}$ is a convex combination of teacher distributions
\item Normalization: $q_{\text{ctx}}(\cdot|x, c)$ is a valid probability distribution for each input $x$ and context $c$
\item Weight conformance: The aggregation respects the specified context-level weights
\end{enumerate}
\end{definition}

\begin{theorem}[Existence and Non-Uniqueness of Context Weight Operators]
\label{thm:ctx-exist}
There exist multiple distinct families of context-level weight functions satisfying Axioms~\ref{ax:ctx-norm}--\ref{ax:ctx-safety}.
\end{theorem}

\begin{proof}[Proof (Sketch)]
Existence follows from broad classes of admissible constructions. For example, context-level weights may be defined as functions of domain identifiers, ordinal safety comparisons on designated contexts, or divergence-based measures of distributional shift, provided these functions satisfy normalization, positivity, boundedness, regularity, and safety monotonicity. Distinct functional dependence yields multiple conforming families, demonstrating non-uniqueness.
\end{proof}

\section{Hierarchical Unification: Product-Structure Composition}
\label{sec:unification}

We unify token-, task-, and context-level adaptive weighting into a single hierarchical framework using a product-structure composition. Within this framework, each scale addresses a distinct source of heterogeneity, and the unified operator aggregates their influence in a mathematically controlled manner.

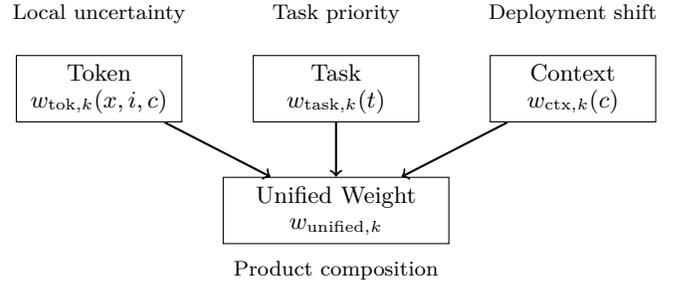
\begin{figure}[t]
\centering
\begin{tikzpicture}[
    scale=0.9,
    box/.style={rectangle, draw, minimum width=2.2cm, minimum height=0.8cm, align=center, font=\small},
    arrow/.style={->, thick}
]
% Three scales
\node[box] (tok) at (0,0) {Token\\$w_{\text{tok},k}(x,i,c)$};
\node[box] (task) at (3.5,0) {Task\\$w_{\text{task},k}(t)$};
\node[box] (ctx) at (7,0) {Context\\$w_{\text{ctx},k}(c)$};

% Product composition
\node[box, minimum width=3cm] (unified) at (3.5,-1.8) {Unified Weight\\$w_{\text{unified},k}$};

% Arrows
\draw[arrow] (tok) -- (unified);
\draw[arrow] (task) -- (unified);
\draw[arrow] (ctx) -- (unified);

% Labels
\node[above=0.3cm of tok, font=\footnotesize] {Local uncertainty};
\node[above=0.3cm of task, font=\footnotesize] {Task priority};
\node[above=0.3cm of ctx, font=\footnotesize] {Deployment shift};
\node[below=0.1cm of unified, font=\footnotesize] {Product composition};
\end{tikzpicture}
\caption{Product-structure composition of adaptive weight operators. Each scale (token, task, context) defines a conforming weight operator addressing one dimension of heterogeneity. The unified weight operator inherits axiom satisfaction from its components via multiplicative composition, preserving normalization and boundedness guarantees.}
\label{fig:hierarchy}
\end{figure}

\subsection{Unified Weight Operator}

\begin{definition}[Unified Hierarchical Weight]
\label{def:unified-weight}
Define the unnormalized product weight as:
$$\tilde{w}_{\text{unified},k}(x, i, t, c) = w_{\text{tok},k}(x, i, c) \cdot w_{\text{task},k}(t) \cdot w_{\text{ctx},k}(c)$$
The unified hierarchical weight is obtained via normalization:
$$w_{\text{unified},k}(x, i, t, c) = \frac{\tilde{w}_{\text{unified},k}(x, i, t, c)}{\sum_{j=1}^K \tilde{w}_{\text{unified},j}(x, i, t, c)}$$
This definition satisfies three properties:
\begin{enumerate}
\item Product structure: The unified weight decomposes as a normalized product of three scale-specific components
\item Normalization preservation: For all $(x, i, t, c)$, the $K$ unified weights sum to 1
\item Bounded composition: If component weights are bounded by $[w_{\min}, w_{\max}]$, unified weights are bounded
\end{enumerate}
\end{definition}

\begin{theorem}[Validity of Product Composition]
\label{thm:product-valid}
Let $w_{\text{tok}}$, $w_{\text{task}}$, $w_{\text{ctx}}$ satisfy their respective axioms. Define the unified weight via product structure as in Definition~\ref{def:unified-weight}. Then:
\begin{enumerate}
\item Normalization: The unified weights define valid probability distributions over teachers
\item Boundedness: Unified weights satisfy bounded influence constraints
\item Convexity: The unified ensemble operator $G_{\text{unified}}$ produces valid convex combinations of teacher distributions
\end{enumerate}
\end{theorem}

\begin{proof}[Proof (Sketch)]
Normalization follows directly from explicit renormalization of the product weights. Boundedness follows from multiplicative composition of bounded functions and subsequent normalization. Convexity and ensemble validity inherit from the component ensemble operators, since the unified operator remains a convex combination of teacher distributions.
\end{proof}

\subsection{Scale Separation Property}

\begin{proposition}[Additive Log-Space Decomposition]
\label{prop:log-decomp}
In log-space, the unnormalized unified weight admits additive decomposition:
\begin{align*}
\log \tilde{w}_{\text{unified},k}(x, i, t, c) &= \log w_{\text{tok},k}(x, i, c) \\
&\quad + \log w_{\text{task},k}(t) + \log w_{\text{ctx},k}(c)
\end{align*}
This additive structure enables:
\begin{itemize}
\item Modular analysis: Each scale can be studied independently prior to normalization
\item Compositional design: Subsets of scales (e.g., token-only or task + context) can be composed consistently within the same framework
\item Structured optimization analysis: Under standard smoothness assumptions, gradients admit a decomposed form in log-space, supporting scale-aware analysis without asserting full optimization independence
\end{itemize}
\end{proposition}

\section{Existence, Non-Uniqueness, and Construction Principles}
\label{sec:existence}

This section establishes that adaptive weighting operators satisfying the proposed axiomatic framework exist and that the axioms do not uniquely determine a single implementation.

\subsection{Existence via Construction Principles}

\begin{theorem}[Existence of Conforming Operators]
\label{thm:exist-main}
There exist non-trivial operator families $G = (G_{\text{tok}}, G_{\text{task}}, G_{\text{ctx}}, G_{\text{unified}})$ satisfying all axioms at the token, task, and context levels, together with hierarchical product composition.
\end{theorem}

\begin{proof}[Proof (Sketch)]
Existence follows from the closure of broad classes of admissible weighting functions under normalization and bounded composition. We outline three representative construction classes.

Principle 1 -- Uncertainty-based constructions. Operators that map teacher uncertainty measures (e.g., entropy, variance, disagreement) to weights via monotone decreasing, bounded, and continuous functions satisfy the token-, task-, and context-level axioms. Ordinal safety axioms can be satisfied by incorporating safety-based comparisons on designated subsets.

Principle 2 -- Performance-based constructions. Operators defined using task- or context-specific performance measures (such as validation loss or accuracy), combined with normalization and boundedness constraints, satisfy all axioms without fixing a particular functional form.

Principle 3 -- Information-theoretic constructions. Operators derived from optimizing weighted divergence objectives between teacher and student distributions, under standard positivity, boundedness, and normalization constraints, yield conforming weight families. Different choices of divergence or constraint structure produce distinct admissible operators.

Each construction class yields at least one valid operator family, establishing existence.
\end{proof}

\subsection{Non-Uniqueness and Implementation Diversity}

\begin{theorem}[Non-Uniqueness of Conforming Operators]
\label{thm:nonunique-main}
The full collection of axioms across all three scales, together with hierarchical product composition, does not uniquely determine the adaptive weighting operators. Multiple distinct operator families satisfy all axioms.
\end{theorem}

\begin{proof}[Proof (Sketch)]
Consider three families of operators with different functional dependence.

Family A (Entropy-based). Token-level weights decrease exponentially with predictive entropy, task-level weights are proportional to inverse task loss, and context-level weights are based on ordinal safety comparisons.

Family B (Variance-based). Token-level weights decrease with prediction variance, task-level weights depend on gradient alignment, and context-level weights adapt via domain-specific measures.

Family C (Hybrid). Token-level weights combine uncertainty and safety indicators, task-level weights use performance-based softmax weighting, and context-level weights depend on distributional distance.

Each family satisfies all axioms and product composition while producing different numerical weight assignments for the same teacher ensemble, demonstrating non-uniqueness.
\end{proof}

\begin{remark}[Implementation Flexibility]
\label{rem:flexibility}
Non-uniqueness is an inherent feature of the axiomatic framework. By specifying structural constraints rather than a single functional form, the framework permits multiple valid implementations that share theoretical guarantees while differing in construction details. This flexibility supports methodological innovation and adaptation to heterogeneous application settings without requiring disclosure of specific implementation choices.
\end{remark}

\section{Convergence Guarantees (Operator-Agnostic)}
\label{sec:convergence}

We establish convergence guarantees for gradient-based optimization of adaptive-weighted knowledge distillation objectives without specifying explicit weight formulas. The results apply uniformly to all weight operators satisfying the axiomatic framework.

\subsection{Adaptive KD Objective (Abstract Formulation)}

\begin{definition}[Adaptive KD Loss]
\label{def:adaptive-loss}
For any conforming weight operator $G$ satisfying the axioms, define the adaptive KD loss $\mathcal{L}_{\text{KD}}(\theta; G)$ as the expected cross-entropy between:
\begin{itemize}
\item The weighted teacher ensemble $q(\cdot|x, t, c)$ produced by $G$
\item The student distribution $p^{(S)}(\cdot|x; \theta)$
\end{itemize}
averaged over inputs, tasks, and contexts.
\end{definition}

\subsection{Main Convergence Theorem}

\begin{theorem}[Operator-Agnostic Convergence]
\label{thm:convergence-main}
Let $G$ be any weight operator satisfying all axioms and product composition. Assume (A1)--(A4) hold. Let $\{\theta_t\}_{t=0}^\infty$ be the sequence of student parameters generated by stochastic gradient descent on $\mathcal{L}_{\text{KD}}(\theta; G)$ with learning rate $\eta_t = \eta_0/(1+t)$. Then, under standard stochastic approximation assumptions:

(i) Almost sure convergence: The loss sequence $\{\mathcal{L}_{\text{KD}}(\theta_t; G)\}$ converges almost surely to a limit value $\mathcal{L}^*$.

(ii) KL divergence vanishing: Any accumulation point $\theta^*$ of the sequence $\{\theta_t\}$ satisfies:
$$\E_{x,t,c}[\KL(q(\cdot|x,t,c) \| p^{(S)}(\cdot|x; \theta^*))] = 0$$

(iii) Convergence rate: Under strong convexity of the loss (e.g., with $\ell_2$ regularization), the convergence rate is $O(1/t)$:
$$\E[\mathcal{L}_{\text{KD}}(\theta_t; G)] - \mathcal{L}^* \leq \frac{C}{t}$$
for some constant $C$ depending on problem parameters.
\end{theorem}

\begin{proof}[Proof (Sketch)]
By assumption (A1), weights are bounded, ensuring the weighted ensemble $q$ remains a valid convex combination of teacher distributions. By assumption (A4), the loss is smooth in $\theta$. The learning rate schedule $\eta_t = O(1/t)$ satisfies the Robbins-Monro conditions ($\sum_t \eta_t = \infty$, $\sum_t \eta_t^2 < \infty$), which guarantee almost sure convergence of stochastic gradient descent to a stationary point. At stationarity, the vanishing gradient condition implies KL divergence minimization, which is uniquely achieved when the student distribution matches the weighted ensemble target.
\end{proof}

\begin{remark}[Key Insight]
\label{rem:convergence-insight}
Theorem~\ref{thm:convergence-main} establishes that convergence guarantees depend only on structural properties of the weighting operator, not on its specific functional form. Any conforming adaptive weighting scheme inherits the same asymptotic convergence guarantees under the framework's assumptions.
\end{remark}

\begin{remark}[Approximate Convergence Under Weight Estimation Noise]
\label{rem:approximate-convergence}
In practice, weight operators must be estimated from finite data, introducing perturbations $\epsilon_t$ such that observed weights $\tilde{w}_t = w_t + \epsilon_t$ differ from ideal weights. Suppose $\|\epsilon_t\|_\infty \leq \delta$ for some $\delta > 0$ and the bounded-weight assumption (A1) holds with margin: $w_{\min} + \delta \leq \tilde{w}_k \leq w_{\max} - \delta$. Then convergence degrades gracefully, with the loss converging to within $O(\delta \cdot L)$ of the optimal value, where $L$ is the Lipschitz constant from assumption (A2). This bound characterizes the robustness-accuracy trade-off for practical implementations with estimated weights.
\end{remark}

\begin{remark}[Non-Uniform Constants Across Operators]
\label{rem:nonuniform-constants}
While Theorem~\ref{thm:convergence-main} guarantees an $O(1/t)$ convergence rate for all conforming operators, the associated constants may vary across operator realizations. Different admissible weight functions can induce different conditioning and curvature properties. The axiomatic framework ensures existence of a valid asymptotic rate for each conforming operator, but does not assert uniform constants across all operators.
\end{remark}

\subsection{Relation to Uniform Weighting}

\begin{proposition}[Uniform KD as Special Case]
\label{prop:uniform-special}
Setting $w_{\text{tok},k} = w_{\text{task},k} = w_{\text{ctx},k} = 1/K$ for all $k$ recovers classical uniform-weight KD. Theorem~\ref{thm:convergence-main} reduces to standard KD convergence guarantees in this limit.
\end{proposition}

\begin{remark}[No Asymptotic Rate Degradation]
\label{rem:no-penalty}
Adaptive weighting does not degrade the asymptotic $O(1/t)$ convergence rate relative to uniform weighting. Any differences appear only in operator-dependent constants, not in the order of convergence.
\end{remark}

\begin{remark}[Extension to Other Divergences]
\label{rem:other-divergences}
Although the analysis is presented using cross-entropy (equivalently, KL divergence), analogous convergence results extend to other $f$-divergences under corresponding smoothness and curvature assumptions. The axiomatic framework characterizes weight operator properties independently of the specific divergence employed.
\end{remark}

\begin{remark}[Role of Weight Bounds]
\label{rem:weight-bounds}
The bounds $w_{\min}$ and $w_{\max}$ in assumption (A1) govern a trade-off between stability and specialization. Tighter bounds favor stability, while looser bounds permit stronger adaptation. The framework guarantees convergence for any admissible bound configuration, leaving selection of bounds to application-specific considerations.
\end{remark}

\section{Stability and Perturbation Analysis}
\label{sec:stability}

We characterize the stability of adaptive-weighted knowledge distillation under perturbations of the weighting operator. The results describe sufficient conditions under which converged solutions vary continuously with respect to changes in weights and training noise.

\subsection{Fixed-Point Characterization Under Contractive Updates}

\begin{theorem}[Fixed-Point Existence via Contraction Mapping]
\label{thm:fixed-point}
Let $\mathcal{W}$ denote the space of admissible weight functions satisfying assumptions (A1)--(A2), equipped with a suitable norm. Consider a weight-update operator $\mathcal{T}$ mapping $\mathcal{W}$ to itself, defined abstractly via feedback from training dynamics. Suppose that, for sufficiently small update step sizes, this operator is a contraction mapping on $\mathcal{W}$ with contraction constant $\rho < 1$.

(i) Contraction property: For sufficiently controlled update step sizes $\beta < 1/L$, $\mathcal{T}$ is a contraction mapping on $\mathcal{W}$ with constant $\rho = \beta L < 1$.

(ii) Unique fixed point: There exists a unique fixed point $w^* \in \mathcal{W}$ such that $\mathcal{T}(w^*) = w^*$.

(iii) Geometric convergence: For any initial weights $w^{(0)} \in \mathcal{W}$, iterates $w^{(n+1)} = \mathcal{T}(w^{(n)})$ converge geometrically:
$$\|w^{(n)} - w^*\|_\infty \leq \rho^n \|w^{(0)} - w^*\|_\infty$$
\end{theorem}

\begin{proof}[Proof (Sketch)]
The update operator contracts distances in weight space due to Lipschitz continuity (A2) and bounded influence (A1). For step sizes $\beta < 1/L$ where $L$ is the Lipschitz constant from assumption (A2), the operator satisfies $\|\mathcal{T}(w) - \mathcal{T}(w')\|_\infty \leq \rho \|w - w'\|_\infty$ with $\rho = \beta L < 1$. The Banach fixed-point theorem then guarantees existence, uniqueness, and geometric convergence to the fixed point.
\end{proof}

\subsection{Perturbation Robustness}

\begin{theorem}[Robustness to Weight Perturbations]
\label{thm:perturbation}
Let $w$ and $\tilde{w}$ be two weight configurations satisfying assumption (A1) with perturbation bound $\|w - \tilde{w}\|_\infty \leq \delta$. Denote by $\theta^*_w$ and $\theta^*_{\tilde{w}}$ the corresponding optimal student parameters obtained by minimizing $\mathcal{L}_{\text{KD}}(\theta; w)$ and $\mathcal{L}_{\text{KD}}(\theta; \tilde{w})$ respectively. Then:
$$\|\theta^*_w - \theta^*_{\tilde{w}}\| \leq C \cdot \delta$$
for some constant $C$ depending on problem parameters ($K$, $V$, $M$, and the strong convexity modulus of the loss).
\end{theorem}

\begin{proof}[Proof (Sketch)]
The difference in optimal parameters is controlled by the difference in gradients at optimality. Under strong convexity of the loss and Lipschitz continuity of the weight operator (A2), parameter deviation is bounded linearly in terms of weight perturbation magnitude.
\end{proof}

\begin{remark}[Graceful Degradation]
\label{rem:graceful}
Theorem~\ref{thm:perturbation} formalizes the notion that small perturbations in adaptive weights lead to proportionally small changes in the converged student solution. This continuity property ensures that estimation noise or approximation error in weight computation does not induce unstable behavior.
\end{remark}

\subsection{Gradient Stability}

\begin{lemma}[Bounded Gradient Variance]
\label{lem:grad-variance}
Under assumptions (A1)--(A4), the variance of stochastic gradients in adaptive-weighted KD is bounded:
$$\text{Var}[\nabla_\theta \mathcal{L}_{\text{KD}}(\theta; G)] \leq \left(\frac{w_{\max}}{w_{\min}}\right)^2 \cdot \sigma^2_{\text{base}}$$
where $\sigma^2_{\text{base}}$ is the intrinsic gradient variance under uniform weighting.
\end{lemma}

\begin{proof}[Proof (Sketch)]
Bounded weights limit the amplification of individual teacher contributions to the gradient. Variance scaling follows from standard variance-propagation bounds under weighted aggregation.
\end{proof}

\begin{remark}[Bounded Variance Implication]
\label{rem:bounded-variance}
Lemma~\ref{lem:grad-variance} shows that adaptive weighting does not induce unbounded amplification of gradient noise relative to uniform weighting. This bounded-variance property supports stable training dynamics and complements earlier analyses of variance-driven instability in knowledge distillation~\cite{paper1_5}.
\end{remark}

\section{Safety-Constrained Adaptive Weighting}
\label{sec:safety}

For high-stakes applications, adaptive weighting must respect explicit safety constraints in addition to predictive performance. This section provides an abstract, operator-agnostic formulation of safety-constrained knowledge distillation and characterizes the resulting optimization structure.

\subsection{Abstract Safety Measures}

\begin{definition}[Safety-Critical Token Set]
\label{def:safety-tokens}
A safety-critical token set $\mathcal{S} \subseteq \V$ is a designated subset of tokens for which elevated reliability is required, such as medical terminology or policy-sensitive content markers.
\end{definition}

\begin{definition}[Safety Measure]
\label{def:safety-measure}
A safety measure $\Safety: \Delta(\V) \times \V \to [0,1]$ is a bounded function that quantifies the reliability of a predictive distribution $p$ on ground-truth label $y^*$. Formally:
$$\Safety(p; y^*) = \begin{cases}
f(p, y^*) & \text{if } y^* \in \mathcal{S} \\
1 & \text{if } y^* \notin \mathcal{S}
\end{cases}$$
where $f: \Delta(\V) \times \mathcal{S} \to [0,1]$ measures reliability on safety-critical tokens.
\end{definition}

\begin{definition}[Expected Safety]
\label{def:expected-safety}
The student's expected safety on distribution $\mathcal{D}$ is defined as:
$$\Safety_{\mathcal{D}}(\theta) := \E_{(x,y^*) \sim \mathcal{D}}[\Safety(p^{(S)}(\cdot|x; \theta); y^*)]$$
\end{definition}

\subsection{Safety-Constrained Optimization}

\begin{definition}[Safety-Constrained KD]
\label{def:safety-constrained}
The safety-constrained knowledge distillation problem is:
$$\min_{\theta \in \Theta} \mathcal{L}_{\text{KD}}(\theta; G) \quad \text{subject to} \quad \Safety_{\mathcal{D}}(\theta) \geq \Safety_{\min}$$
where $\Safety_{\min} \in (0, 1]$ is a user-specified minimum safety threshold and $G$ is a conforming weight operator.
\end{definition}

\subsection{Safety--Performance Trade-offs}

\begin{theorem}[Safety-Performance Pareto Characterization]
\label{thm:pareto}
Consider the bi-objective problem minimizing $(\mathcal{L}_{\text{KD}}, -\Safety_{\mathcal{D}})$. The set of Pareto-optimal solutions can be characterized via the Lagrangian:
$$\mathcal{L}(\theta, \mu) = \mathcal{L}_{\text{KD}}(\theta) - \mu \cdot \Safety_{\mathcal{D}}(\theta), \quad \mu \geq 0$$
Each value of $\mu$ corresponds to a point on the Pareto frontier trading off distillation loss against safety.
\end{theorem}

\begin{proof}[Proof (Sketch)]
By KKT theory for constrained optimization, the safety constraint is active at optimality when $\Safety_{\min}$ lies in the interior of the feasible safety range. The Lagrangian captures the loss-safety trade-off. Strong duality holds under the Slater condition (existence of a strictly feasible point), ensuring that Pareto-optimal solutions are characterized by stationary points of the Lagrangian.
\end{proof}

\subsection{Optimality Conditions}

\begin{theorem}[KKT Conditions for Safety-Constrained Adaptive Weighting]
\label{thm:kkt}
Let $\theta^*$ solve the safety-constrained KD problem (Definition~\ref{def:safety-constrained}). Then there exists a Lagrange multiplier $\mu^* \geq 0$ such that the following conditions hold:
\begin{enumerate}
\item Stationarity: $\nabla_\theta \mathcal{L}(\theta^*, \mu^*) = \nabla_\theta \mathcal{L}_{\text{KD}}(\theta^*) - \mu^* \nabla_\theta \Safety_{\mathcal{D}}(\theta^*) = 0$
\item Complementary slackness: $\mu^* \cdot (\Safety_{\mathcal{D}}(\theta^*) - \Safety_{\min}) = 0$
\item Primal feasibility: $\Safety_{\mathcal{D}}(\theta^*) \geq \Safety_{\min}$
\item Dual feasibility: $\mu^* \geq 0$
\end{enumerate}
\end{theorem}

\subsection{Safety Preservation Under Ordinal Weighting}

\begin{corollary}[Safety Preservation]
\label{cor:safety-preserve}
Assume safety-aware context weighting satisfying Axiom~\ref{ax:ctx-safety} and suppose the safety measure is concave in the predictive distribution. Then, at convergence, the student's expected safety in safety-critical contexts is bounded below by the expected safety of the corresponding weighted teacher ensemble.
\end{corollary}

\begin{proof}[Proof (Sketch)]
By Theorem~\ref{thm:convergence-main}, the student distribution matches the weighted teacher ensemble at convergence. Under concavity of the safety measure, Jensen's inequality implies that the safety of the ensemble lower-bounds the expected safety achieved by the student relative to that ensemble.
\end{proof}

\section{Summary and Implications}
\label{sec:conclusion}

\subsection{Theoretical Contributions}

This paper develops an axiomatic framework for adaptive weighting in multi-teacher knowledge distillation. Section~\ref{sec:preliminaries} establishes notation and standing assumptions under which adaptive weighting operators are well-defined. Section~\ref{sec:axioms} introduces axioms governing token-, task-, and context-level weighting, formalizing structural requirements such as normalization, boundedness, continuity, and ordinal safety monotonicity at each scale.

Section~\ref{sec:unification} shows how these scale-specific operators can be composed hierarchically via product-structure normalization, yielding a unified weighting operator that preserves the axiomatic properties of its components. Section~\ref{sec:existence} establishes that conforming operators exist and are inherently non-unique, demonstrating that the axioms define a broad admissible family rather than a single prescribed implementation.

Building on this foundation, Section~\ref{sec:convergence} characterizes convergence of gradient-based optimization for adaptive distillation objectives under standard stochastic approximation assumptions, independent of the specific choice of weight operator. Section~\ref{sec:stability} analyzes stability and perturbation robustness, showing that bounded and regular weighting ensures graceful degradation under estimation noise and controlled sensitivity to weight perturbations.

Section~\ref{sec:safety} extends the framework to safety-constrained distillation, providing an abstract formulation of safety measures, Pareto trade-offs between performance and safety, and conditions under which adaptive weighting preserves ensemble-level safety properties in safety-critical contexts.

Collectively, these results decouple theoretical guarantees from specific weighting formulas, clarifying which aspects of adaptive distillation are structural and which are implementation-dependent. The framework provides a principled foundation for designing, analyzing, and comparing adaptive weighting strategies under heterogeneity, distribution shift, and safety constraints. While this paper focuses on theoretical characterization, the axiomatic structure naturally motivates empirical evaluation of concrete operator instantiations, which we address in subsequent work.

\subsection{Practical Implications}

For researchers developing novel adaptive weighting strategies, the axiomatic framework provides a set of design principles that ensure theoretical guarantees. Any weight operator satisfying the five axioms at each scale inherits the convergence, stability, and safety preservation results established in this paper. Researchers can focus on empirical performance within this theoretically grounded design space, knowing that fundamental properties are preserved by construction.

For practitioners deploying knowledge distillation systems, the framework enables modular composition of adaptive weighting at multiple scales. Token-level uncertainty weighting can be combined with task-level prioritization and context-level safety routing using the product-then-normalize composition without sacrificing convergence rate or stability. The standing assumptions provide checkable conditions for verifying that a proposed weight operator conforms to the framework.

For safety-critical applications in medical diagnosis, autonomous systems, and content moderation, the safety-constrained formulation provides formal guarantees that designated safety orderings are preserved. Practitioners can specify minimum safety thresholds and verify that weight operators satisfying the safety monotonicity axiom will maintain these thresholds under deployment shift.

\subsection{Open Questions and Future Directions}

Several theoretical and practical questions remain open for future investigation.

The first concerns axiom minimality. The framework specifies five axioms at each of three scales, yielding 15 total constraints. It remains unclear whether this set is minimal or whether a smaller set would suffice to preserve existence, convergence, and stability. Axiom independence can be probed via counterexample construction: for each axiom, one would exhibit a weight family satisfying all other axioms but violating that axiom, thereby demonstrating that the axiom is not redundant. A systematic independence analysis would clarify the essential structure of valid adaptive weighting.

The second concerns computational complexity. The framework establishes that conforming operators exist but does not characterize their computational cost. Information-theoretic lower bounds via Le Cam's method or Fano's inequality could establish whether $O(N \cdot K \cdot V)$ computation is inherent or whether sublinear-in-$V$ approximations are possible. Such bounds would inform practical implementation strategies for large vocabularies.

The third concerns interaction with gradient modification schemes. Multi-objective optimization frameworks such as gradient surgery and Pareto-MTL modify gradients to balance competing objectives. The interaction between adaptive weighting and gradient modification remains unexplored. Extending the axiomatic framework to characterize valid weight-gradient interactions could yield improved multi-task solutions.

The fourth concerns non-stationary distributions. The current framework assumes stationary data distributions during training. In continual learning scenarios where distributions shift over time, adaptive weights must evolve to track the changing environment. The contraction framework suggests that weight update operators could be designed to track distribution drift while maintaining stability, but formal guarantees for this setting remain to be established.

The fifth concerns scale-specific temperature. The temperature parameter in knowledge distillation controls the sharpness of teacher distributions. Making temperature scale-adaptive, with different temperatures for token, task, and context aggregation, could enable finer control over the distillation process. Extending the axiomatic framework to characterize valid temperature operators would complement the weight operator analysis.

The sixth concerns theoretical tightness. The $O(1/t)$ convergence rates and $O(\delta)$ perturbation bounds established in this paper are upper bounds. Lower bound constructions would clarify whether these rates are tight or whether improved analysis could yield sharper guarantees. Matching upper and lower bounds would complete the theoretical characterization of adaptive weighting.

\section*{Acknowledgements}

The authors gratefully acknowledge the collaborative environment at SparseTech that made this research possible. The theoretical and computational developments presented in this paper are part of an ongoing SparseTech research initiative on adaptive knowledge distillation architectures. Patent Pending.

\bibliographystyle{plain}
\bibliography{../sparsetech_references}

\appendix

\section*{Appendix A: Toy Illustration (2-Teacher, 3-Token)}
\label{app:toy}

We demonstrate adaptive weighting benefits on vocabulary $\V = \{a, b, c\}$ with two teachers.

\emph{Setup.} Teacher 1 is confident: $p^{(T_1)} = (0.8, 0.15, 0.05)$ with entropy $H_1 = 0.68$. Teacher 2 is uncertain: $p^{(T_2)} = (0.4, 0.35, 0.25)$ with entropy $H_2 = 1.52$. Ground truth is token $a$.

\emph{Uniform Weighting.} With $w_1 = w_2 = 0.5$:
\begin{align*}
q_{\text{uniform}} &= 0.5 \cdot (0.8, 0.15, 0.05) + 0.5 \cdot (0.4, 0.35, 0.25) \\
&= (0.6, 0.25, 0.15)
\end{align*}
Probability on correct token: $q_{\text{uniform}}(a) = 0.60$.

\emph{Adaptive Weighting.} Using entropy-based weights $w_k \propto 1/H_k$:
$$w_1 = \frac{1/0.68}{1/0.68 + 1/1.52} = 0.69, \quad w_2 = 0.31$$
\begin{align*}
q_{\text{adaptive}} &= 0.69 \cdot (0.8, 0.15, 0.05) + 0.31 \cdot (0.4, 0.35, 0.25) \\
&= (0.68, 0.21, 0.11)
\end{align*}
Probability on correct token: $q_{\text{adaptive}}(a) = 0.68$.

\begin{center}
\small
\begin{tabular}{l|cc}
\toprule
\textbf{Method} & $\boldsymbol{q(a)}$ & \textbf{Improvement} \\
\midrule
Uniform & 0.60 & -- \\
Adaptive & 0.68 & +13\% \\
\bottomrule
\end{tabular}
\end{center}

Adaptive weighting up-weights the confident teacher, improving ensemble quality. Both uniform and adaptive satisfy Axioms 1--5, but adaptive exploits heterogeneity while maintaining convergence guarantees (Theorem~\ref{thm:convergence-main}).

\end{document}